\documentclass{article} %

\usepackage[preprint]{neurips_2024} %
\usepackage[numbers]{natbib}
\usepackage[utf8]{inputenc} %
\usepackage[T1]{fontenc}    %
\usepackage{amsfonts}       %
\usepackage{amsmath}
\usepackage{nicefrac}       %
\usepackage{microtype}      %
\usepackage{amsthm}
\usepackage{amssymb}
\usepackage{hyperref}
\usepackage{url}
\usepackage{xcolor}
\usepackage{algorithm}
\usepackage{algpseudocode}
\usepackage{tikz}
\usepackage[capitalize]{cleveref}
\usepackage{wrapfig}
\usepackage{enumitem}
\usepackage{booktabs}
\usepackage{MnSymbol, bbding,pifont}
\usepackage{lipsum}

\usetikzlibrary{calc}

\title{Model-Based Diffusion for Trajectory Optimization}

\author{%
  Chaoyi Pan$^*$, 
  Zeji Yi$^*$,
  Guanya Shi$^\dagger$, 
  Guannan Qu$^\dagger$ \\
  Carnegie Mellon University \\
  \texttt{\{chaoyip,zejiy,guanyas,gqu\}@andrew.cmu.edu} \\
}

\newcommand{\Yi}[1]{Y^{(#1)}} %
\newcommand{\pYi}[1]{p_{#1}} %
\newcommand{\sYi}[2]{Y^{(#1)}} %
\newcommand{\sba}[1]{\sqrt{\bar{\alpha}_{#1}}}

\newcommand{\osec}[1]{${#1} \, \text{s}$}
\newcommand{\msec}[2]{${#1} \, \text{m}\, {#2} \, \text{s}$}
\DeclareMathOperator*{\argmax}{
arg\,max}
\DeclareMathOperator*{\argmin}{arg\,min}
\newtheorem{theorem}{Theorem}
\newtheorem{defin}[theorem]{Definition}
\newtheorem{prop}[theorem]{Proposition}

\DeclareTextFontCommand{\boldemph}{\bfseries\itshape}

\newcommand{\SA}{Monte Carlo score ascent}

\begin{document}

\maketitle

\begin{abstract}
    Recent advances in diffusion models have demonstrated their strong capabilities in generating high-fidelity samples from complex distributions through an iterative refinement process.
    Despite the empirical success of diffusion models in motion planning and control, the model-free nature of these methods does not leverage readily available model information and limits their generalization to new scenarios beyond the training data (e.g., new robots with different dynamics).
    In this work, we introduce \underline{M}odel-\underline{B}ased \underline{D}iffusion (MBD), an optimization approach using the diffusion process to solve trajectory optimization (TO) problems \boldemph{without data}. The key idea is to explicitly compute the score function by leveraging the model information in TO problems, which is why we refer to our approach as \boldemph{model-based} diffusion. Moreover, although MBD does not require external data, it can be naturally integrated with data of diverse qualities to steer the diffusion process.
    We also reveal that MBD has interesting connections to sampling-based optimization.
    Empirical evaluations show that MBD outperforms state-of-the-art reinforcement learning and sampling-based TO methods in challenging contact-rich tasks. Additionally, MBD's ability to integrate with data enhances its versatility and practical applicability, even with imperfect and infeasible data (e.g., partial-state demonstrations for high-dimensional humanoids), beyond the scope of standard diffusion models.
    Videos and codes: \url{https://lecar-lab.github.io/mbd/}
\end{abstract}

\section{Introduction}

\begin{wrapfigure}{r}{0.68\textwidth}
    \centering
    \includegraphics[width=0.68\textwidth,bb=0 0 750.000 502.500]{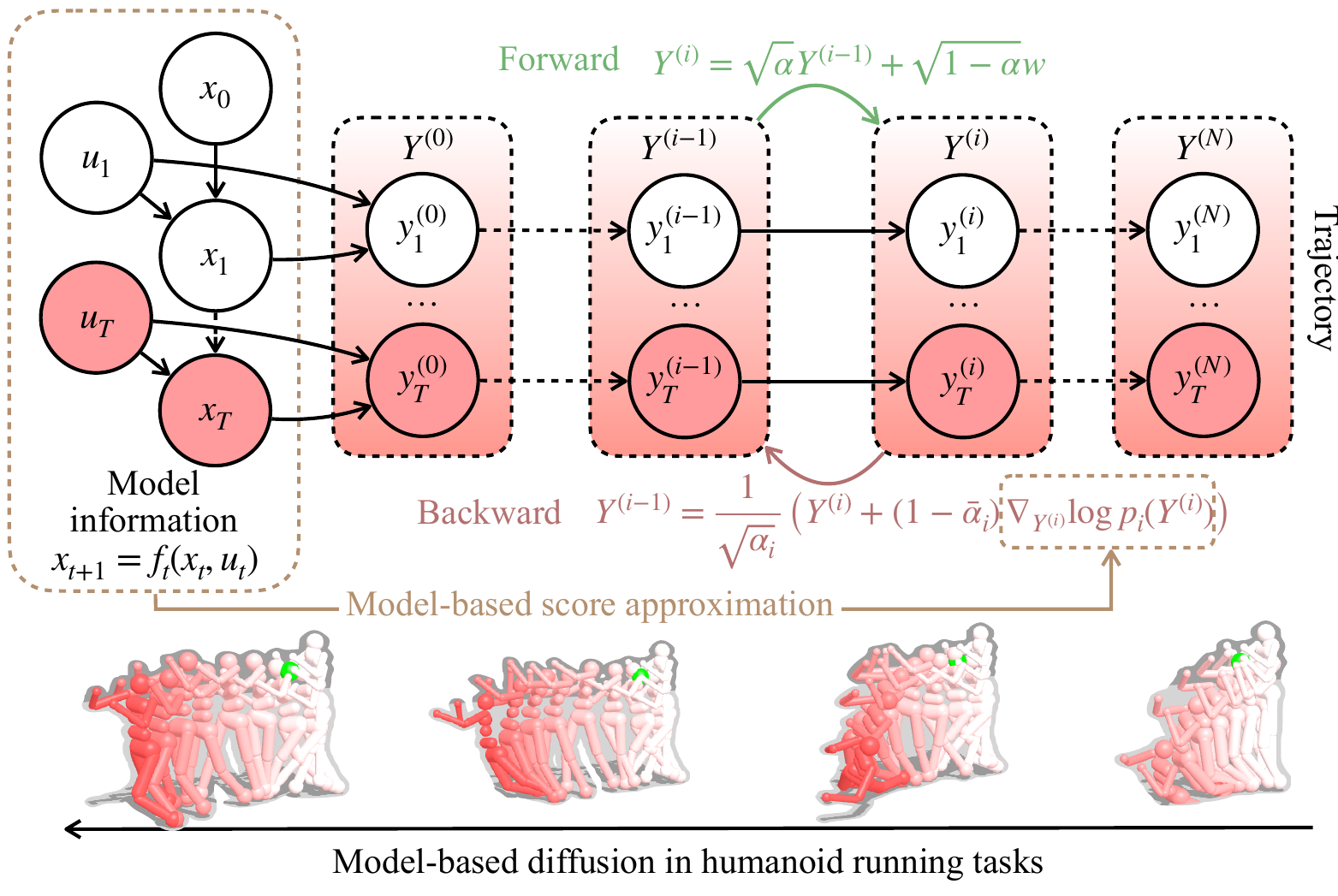}
    \vspace{-0.8cm}
    \caption{MBD refines the trajectory by leveraging the dynamics model directly without relying on demonstration data.}
    \vspace{-0.2cm}
    \label{fig:mbd}
\end{wrapfigure}

Trajectory optimization (TO) aims to optimize the state and control sequence to minimize a cost function while subject to specified dynamics and constraints.
Given nonlinear, non-smooth dynamics and non-convex objectives and constraints, traditional optimization methods like gradient-based methods and interior point methods are less effective in solving TO problems.
In response, diffusion models have emerged as a powerful tool for trajectory generation in complex dynamical systems due to their expressiveness and scalability~\cite{carvalhoMotionPlanningDiffusion2023, tevetHumanMotionDiffusion2022, jiangMotionDiffuserControllableMultiAgent2023, jannerPlanningDiffusionFlexible2022, liangAdaptDiffuserDiffusionModels2023, ajayConditionalGenerativeModeling2023}.

Although diffusion models excel when learning from large-scale, high-dimensional, and high-quality demonstrations, their dependency on such data limits their practicality. 
For example, after training a manipulation task with a specific robotic arm, the model may struggle to generalize to new tasks with a different arm as the underlying dynamics change. 
This limitation arises from the model-free nature of existing diffusion-based methods, which do not leverage readily available model information to enhance adaptability. Moreover, existing diffusion-based approaches often require high-quality (in terms of optimality and feasibility) demonstration data, which limits their applications in various scenarios with imperfect data, such as dynamically infeasible trajectories (e.g., generated by high-level planners using simplified models) and partial demonstrations (e.g., lower-body-only demonstrations for a high-dimensional humanoid).

Fortunately, unlike diffusion model's applications in vision or language where data is from unknown distributions (e.g., internet-scale image data), in trajectory optimization, we often know the distribution of desired trajectories, which is described by the optimization objectives, constraints, and the underlying dynamics model, although such a distribution is intractable to directly sample from.
Diffusion models offer a tantalizing new perspective, by iteratively refining samples from isotropic Gaussians to meaningful desired distributions in manageable steps, rather than directly learning the complex desired distribution.
Inspired by this, we propose Model-Based Diffusion (MBD) that utilizes model information to approximate the gradient of the log probability density function (a.k.a. score function) and uses it to iteratively refine sampled trajectories to solve TO problems, as depicted in~\cref{fig:mbd}. 
This model-centric strategy allows for the generation of dynamically feasible trajectories in a data-free manner, and gradually moves them towards more optimal solutions. 
Furthermore, by using demonstrations as observations of the target distribution, MBD can be smoothly combined with data of different qualities to steer the diffusion process and enhance its effectiveness.
Particularly, we merge the demonstration data into the sampling process by evaluating their likelihoods with the model and use them to improve the estimation of the score function.
Our contributions are threefold:

\begin{itemize}[leftmargin=*]
    \label{item:contributions}
    \item We introduce the Model-Based Diffusion (MBD) framework for trajectory optimization, utilizing the dynamics model to estimate the score function. This enables an effective trajectory planner given non-smooth dynamics and non-convex objectives, such as contact-rich manipulation tasks or high-dimensional humanoids.
    \item Our analysis and empirical evaluations demonstrate that MBD matches, and often exceeds, the performance of existing reinforcement learning and sampling-based TO methods. 
    In particular, MBD outperforms PPO by 59\% in various tasks within tens of seconds of diffusing.
    \item 
    We demonstrate MBD's flexibility in utilizing diverse imperfect data to steer the diffusion process and further enhance the performance. Specifically, the resulting whole-body humanoid trajectory from MBD is more natural by utilizing the lower-body-state-only human motion data.
    Similarly, MBD can effectively address long-horizon sparse-reward Umaze navigation tasks by leveraging infeasible demonstrations generated by an RRT planner with simplified dynamics.
\end{itemize}

\section{Related Work}

\textbf{Diffusion Models.}
Diffusion models have been widely adopted as generative models for high-dimensional data, such as image~\cite{songDenoisingDiffusionImplicit2020}, audio~\cite{chenWaveGradEstimatingGradients2020}, and text~\cite{austinStructuredDenoisingDiffusion2023} through iterative refinement processes~\cite{sohl-dicksteinDeepUnsupervisedLearning2015, hoDenoisingDiffusionProbabilistic2020}.
The backward process can be viewed as gradient prediction~\cite{songGenerativeModelingEstimating2020} or score matching~\cite{songScoreBasedGenerativeModeling2021}, which learns the score function to move samples towards the data distribution.
We deliver new methods to perform the backward diffusion process using the available model information.

\textbf{Sampling-based Optimization.} Optimization involving black-box functions is widely applied across various fields, including hyperparameter tuning and experimental design \cite{snoekPracticalBayesianOptimization2012, hernandez-lobatoParallelDistributedThompson2017}.
Evolutionary algorithms like CMA-ES are often used to tackle black-box optimization problems, dynamically modifying the covariance matrix to produce new samples \cite{hansenReducingTimeComplexity2003}. Such problems can also be efficiently addressed within the Bayesian optimization framework \cite{sohl-dicksteinDeepUnsupervisedLearning2015, frazierTutorialBayesianOptimization2018}, which offers greater efficiency.
Nonetheless, traditional BO algorithms are generally restricted to low-dimensional problems.

\textbf{Trajectory Optimization.}
Traditionally, trajectory optimization (TO) is solved using gradient-based optimization, which faces challenges such as non-convex problem structures, nonlinear or discontinuous dynamics, and high-dimensional state and control action spaces.
As two equivalent formulations, direct methods~\cite{hargravesDirectTrajectoryOptimization2012} and shooting-based methods~\cite{howellALTROFastSolver2019} are commonly used to solve TO problems, where gradient-based optimizers such as Augmented Lagrangian~\cite{jalletConstrainedDifferentialDynamic2022}, Interior Point~\cite{kimInteriorPointMethodLargeScale2007}, and Sequential Quadratic Programming~\cite{nocedalSequentialQuadraticProgramming2006,shiNeuralSwarm2PlanningControl2022} are employed.
To leverage the parallelism of modern hardware and improve global convergence properties, sampling-based methods like Cross-Entropy Motion Planning (CEM) \cite{kobilarovCrossentropyMotionPlanning2012} and Model Predictive Path Integral (MPPI) \cite{williamsModelPredictivePath2017, yiCoVOMPCTheoreticalAnalysis2024} have been proposed to solve TO by sampling from target distributions.
To solve stochastic optimal control problems, trajectory optimization has also been framed as an inference problem in a probabilistic graphical model, where system dynamics defines the graph structure \cite{kappenOptimalControlGraphical2012,levineReinforcementLearningControl2018}.
This perspective extends methods such as iLQG by integrating approximate inference techniques to improve trajectory optimization \cite{toussaintRobotTrajectoryOptimization2009}.
The connection between diffusion and optimal control has been explored in \cite{bernerOptimalControlPerspective2023}, which motivates us to use diffusion models as solvers for trajectory optimization.

\textbf{Diffusion for Planning.}
Diffusion-based planners have been used to perform human motion generation \cite{carvalhoMotionPlanningDiffusion2023, tevetHumanMotionDiffusion2022} and multi-agent motion prediction \cite{jiangMotionDiffuserControllableMultiAgent2023}.
Diffusion models are capable of generating complete trajectories by folding both dynamics and optimization processes into a single framework, thus mitigating compounding errors and allowing flexible conditioning \cite{jannerPlanningDiffusionFlexible2022, liangAdaptDiffuserDiffusionModels2023, ajayConditionalGenerativeModeling2023}. In addition, they have been adeptly applied to policy generation, enhancing the capability to capture multimodal demonstration data in high-dimensional spaces for long-horizon tasks \cite{reussGoalConditionedImitationLearning2023, chiDiffusionPolicyVisuomotor2023}.
These works assume no access to the underlying dynamics, limiting the generalization to new environments.
To enforce dynamics constraints, SafeDiffuser~\cite{xiaoSafeDiffuserSafePlanning2023} integrates control barrier functions into the learned diffusion process, while Diffusion-CCSP~\cite{yangCompositionalDiffusionBasedContinuous2023} composes the learned geometric and physical conditions to guarantee constraint compliance.
Our approach uses diffusion models directly as solvers, rather than simply distilling solutions from demonstrations.

\section{Problem Statement and Background}

\textbf{Notations}:
We use lower (upper) scripts to specify the time (diffusion) step, e.g., $x_t, u_t, y_t$ represent the state, control and state-control pair at time $t$, and $Y^{(i)}$ represents the diffusion state at step $i$.

This paper focuses on a class of trajectory optimization (TO) problems whose objective is to find the sequences $ \{x_t\} $ and $ \{u_t\} $ that minimize the cost function $ J(x_{1:T} ; u_{1:T}) $ subject to the dynamics and constraints.
The optimization problem \footnote{We assume deterministic dynamics for simplicity to sample the dynamically feasible trajectory. The extension to stochastic dynamics is straightforward.}
can be formulated as follows:
\begin{subequations}\label{eq:to_problem}
    \begin{align}
        \min_{x_{1:T}, u_{1:T}} & J(x_{1:T} ; u_{1:T})= l_T(x_T) + \sum_{t=0}^{T-1} l_t(x_t, u_t) \label{eq:to_obj}            \\
        \text{s.t.}   \quad     & x_0                         = x_{\text{init}}                                                \\
                                & x_{t+1}   = f_t(x_t, u_t), \quad \forall t = 0, 1, \dots, T-1,             \label{eq:to_dyn} \\
                                & g_t(x_t, u_t)              \leq 0, \quad \forall t = 0, 1, \dots, T-1.
    \end{align}
\end{subequations}

where $ x_t \in \mathbb{R}^{n_x} $ and $ u_t \in \mathbb{R}^{n_u} $ are the state and control at time $ t $, $ f_t : \mathbb{R}^{n_x} \times \mathbb{R}^{n_u} \rightarrow \mathbb{R}^{n_x} $ represents the dynamics, $ g_t : \mathbb{R}^{n_x} \times \mathbb{R}^{n_u} \rightarrow \mathbb{R}^{n_g} $ are the constraint functions and $ l_t : \mathbb{R}^{n_x} \times \mathbb{R}^{n_u} \rightarrow \mathbb{R} $ are the stage costs.
We use $Y = [x_{1:T} ; u_{1:T}]$ to denote all decision variables. 
Traditionally, TO is solved using nonlinear programming, which faces challenges such as non-convex problem structures, nonlinear or discontinuous dynamics, and high-dimensional state and control action spaces.
Recently, there has been a growing interest in bypassing these challenges by directly generating samples from the optimal trajectory distribution using diffusion models trained on optimal demonstration data~\cite{carvalhoMotionPlanningDiffusion2023, liangAdaptDiffuserDiffusionModels2023, reussGoalConditionedImitationLearning2023, yangCompositionalDiffusionBasedContinuous2023}.

To use diffusion for TO, \eqref{eq:to_problem} is first transformed into a sampling problem. 
The target distribution $\pYi{0}(\Yi{0})$ is proportional to dynamical feasibility $p_d(Y) \propto \prod_{t=1}^{T} \mathbf{1}(x_t = f_{t-1}(x_{t-1},u_{t-1}))$, optimality $p_J(Y) \propto \exp{(-\frac{J(Y)}{\lambda})}$ and the constraints $p_g(Y) \propto \prod_{t=1}^{T} \mathbf{1}(g_t(x_t, u_t)\leq 0)$, i.e.,
\begin{equation}
    \pYi{0}(Y) \propto p_d(Y)p_J(Y)p_g(Y) \label{eq:p0}
\end{equation}
Obtaining the solution $Y^*$ from the TO problem in~\cref{eq:to_problem} is equivalent to sampling from~\cref{eq:p0} given a low temperature $\lambda \rightarrow 0$. 
In fact, in \cref{sec:Apd_low_t}, we prove that the distribution of $J(Y)$ with $Y\sim p_0(\cdot)$ converges in probability to the optimal value $J^*$ as $\lambda \rightarrow 0$, under mild assumptions.
However, it is generally difficult to directly sample from the high-dimensional and sparse target distribution $\pYi{0}(\cdot)$.
To address this issue, the diffusion process iteratively refines the samples following a backward process, which reverses a predefined forward process as shown in~\cref{fig:mbd}.
The forward process corrupts the original distribution $\pYi{0}(\cdot)$ to an isotropic Gaussian $\pYi{N}(\cdot)$ by incrementally adding small noise to it and scaling it down by $\sqrt{\alpha_i}$ to maintain an invariant noise covariance (see~\cref{fig:distributions}(b) for an example). Mathematically, this means we iteratively obtain $\Yi{i}\sim p_i(\cdot)$ with
$
    p_{i|i-1}(\cdot| \Yi{i-1}) \sim \mathcal{N}(\sqrt{\alpha_i} \Yi{i-1}, \sqrt{1-\alpha_i} I)
$.
Because the noise at each time step is independent, the conditional distribution of $\Yi{i}|\Yi{i-1}$ also leads to that of $\Yi{i}|\Yi{0}$:
\begin{align}
    p_{i|0}(\cdot | \Yi{0}) & \sim \mathcal{N}(\sqrt{\bar{\alpha}_i} \Yi{0}, \sqrt{1-\bar{\alpha}_i} I), \quad \bar{\alpha}_i = \prod_{k=1}^{i} \alpha_k. \label{eq:forward}
\end{align}
The backward process $\pYi{i-1|i}(\Yi{i-1}|\Yi{i})$ is the reverse of the forward process $\pYi{i|i-1}(\Yi{i}|\Yi{i-1})$, which removes the noise from the corrupted distribution $\pYi{N}(\cdot)$ to obtain the target distribution $\pYi{0}(\cdot)$.
The target distribution $\pYi{0}(\cdot)$ in the diffusion process is given by:
\begin{align}
    \pYi{i-1}(\Yi{i-1}) & = \int \pYi{i-1|i}(\Yi{i-1}|\Yi{i}) \pYi{i}(\Yi{i}) \, d\Yi{i},               \\
    \pYi{0}(\Yi{0})     & = \int \pYi{N}(\Yi{N}) \prod_{i=N}^{1} \pYi{i-1|i}(\Yi{i-1}|\Yi{i}) d\Yi{1:N}
\end{align}
Standard diffusion models~\cite{jannerPlanningDiffusionFlexible2022, liangAdaptDiffuserDiffusionModels2023,yangCompositionalDiffusionBasedContinuous2023}, which we refer to as Model-Free Diffusion (MFD), solve the backward process by learning score function merely from data. 
In contrast, we propose leveraging the dynamics model to estimate the score to improve the generalizability of the model and allow a natural integration with diverse quality data.

\section{Model-Based Diffusion}

\label[section]{sec:algo}

In this section, we formally introduce our MBD algorithm that leverages model information to perform backward process.
To streamline the discussion, in~\cref{sec:algo_unconstrained}, we first present MBD with \SA{} to solve simplified and generic unconstrained optimization problems. %
In~\cref{sec:algo_to}, we extend MBD to the constrained optimization setting to solve the TO problem given complex dynamics and constraints.
Lastly, in~\cref{sec:algo_demo}, we augment the MBD algorithm with demonstrations to improve sample quality and steer the diffusion process.

\subsection{Model-based Diffusion as Multi-stage Optimization}
\label[section]{sec:algo_unconstrained}

\begin{figure}[H]
    \centering
    \includegraphics[width=1.0\textwidth, bb = 0 0 839.515 201.600]{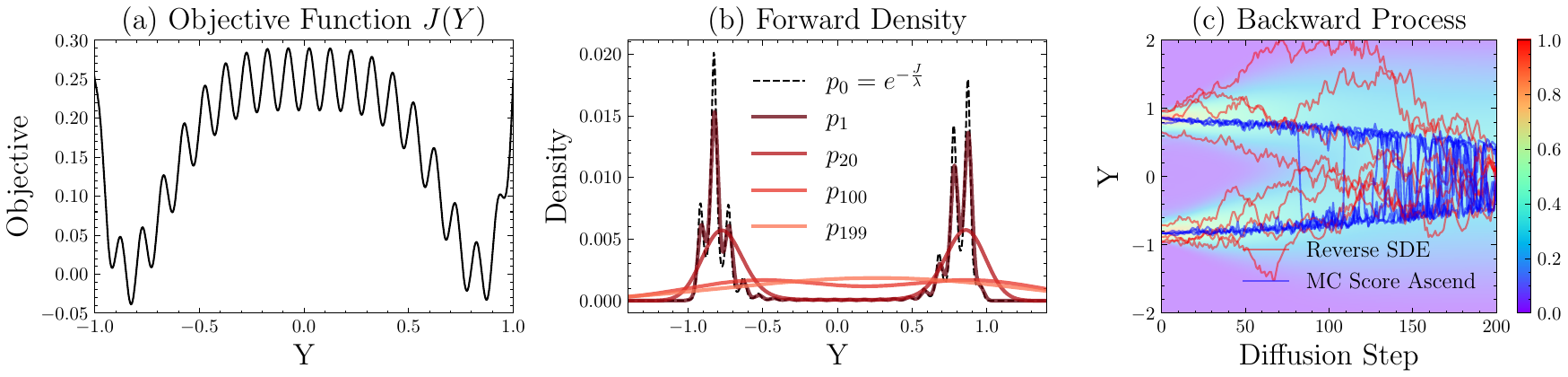}
    \vspace{-0.4cm}
    \caption{Reverse SDE vs. \SA{} (MCSA) on a synthetic highly non-convex objective function. 
    (a) Synthesized objective function with multiple local minima. 
    (b) The intermediate stage density $\pYi{i}(\cdot)$, where peaked $\pYi{0}(\cdot)$ is iteratively corrupted to a Gaussian $\pYi{N}(\cdot)$. 
    (c) Reverse SDE vs. MCSA: Background colors represent the density of $\pYi{i}(\cdot)$ at different stages. MCSA converges faster due to larger step size and lower sampling noise while still capturing the multimodality.} %
    \label{fig:distributions}
\end{figure}

We consider the reverse process for an unconstrained optimization problem $\min_Y J(Y)$, where the target distribution is $\pYi{0}(\Yi{0}) \propto \exp{(-\frac{J(\Yi{0})}{\lambda})}$.
In our MBD framework, ``model'' implies that we can evaluate $J(\Yi{0})$ for arbitrary $\Yi{0}$, enabling us to compute the target distribution up to the normalizing constant.%

MBD uses \SA{} instead of the commonly adopted reverse SDE approach in MFD.
Specifically, when denoising from $i$ to $i-1$, MBD performs one step of gradient ascent on $\log{\pYi{i}(\Yi{i})}$ and then scales the sample by the factor $\frac{1}{\sqrt{\alpha_{i}}}$ as defined in the forward process:
\begin{align}
    \Yi{i-1} =  \frac{1}{\sqrt{\alpha_{i}}}\left(\Yi{i} + (1-\bar{\alpha}_{i})\nabla_{\Yi{i}}\log{\pYi{i}(\Yi{i})}\right)
    \label[equation]{eq:reverse}
\end{align}
Critically, with the model-informed $\pYi{0}(\Yi{0})$, we can estimate the score function $\nabla_{\Yi{i}}\log{\pYi{i}(\Yi{i})}$ by connecting $\pYi{i}(\Yi{i})$ to $\pYi{0}(\Yi{0})$ via Bayes' rule:

\begin{subequations}
    {\small
        \begin{align}
            \nabla_{\Yi{i}}\log{\pYi{i}(\Yi{i})}                                                                                                                                                                                                        = & \frac{\nabla_{\Yi{i}}\int  \pYi{i\mid 0 }(\Yi{i} \mid \Yi{0}) \pYi{0}(\Yi{0}) \, d\Yi{0}}{\int  \pYi{i\mid 0 }(\Yi{i} \mid \Yi{0}) \pYi{0}(\Yi{0}) \, d\Yi{0}}                                                                                                  %
            =                                                                                                                                                                                                                                             \frac{\int  \nabla_{\Yi{i}}\pYi{i\mid 0 }(\Yi{i} \mid \Yi{0}) \pYi{0}(\Yi{0}) \, d\Yi{0}}{\int  \pYi{i\mid 0 }(\Yi{i} \mid \Yi{0}) \pYi{0}(\Yi{0}) \, d\Yi{0}}                                                                         \label{eq:score_b}         \\
            =                                                                                                                                                                                                                                             & \frac{\int -\frac{\Yi{i}-\sqrt{\bar{\alpha}_{i}}\Yi{0}}{1-\bar{\alpha}_{i}}\pYi{i\mid 0 }(\Yi{i} \mid \Yi{0}) \pYi{0}(\Yi{0}) \, d\Yi{0}}{\int  \pYi{i\mid 0 }(\Yi{i} \mid \Yi{0}) \pYi{0}(\Yi{0}) \, d\Yi{0}}                         \label{eq:score_c}       \\
            =                                                                                                                                                                                                                                             & -\frac{\Yi{i}}{1-\bar{\alpha}_{i}} + \frac{\sqrt{\bar{\alpha}_{i}}}{1-\bar{\alpha}_{i}} \frac{\int \Yi{0}\pYi{i\mid 0 }(\Yi{i} \mid \Yi{0}) \pYi{0}(\Yi{0}) \, d\Yi{0}}{\int  \pYi{i\mid 0 }(\Yi{i} \mid \Yi{0}) \pYi{0}(\Yi{0}) \, d\Yi{0}} \label{eq:score_d}
        \end{align}}
\end{subequations}

Between \cref{eq:score_b} and \cref{eq:score_c}, we use the forward Gaussian density in \cref{eq:forward}: $\pYi{i\mid 0 }(\Yi{i} \mid \Yi{0})\propto \exp(-\frac{1}{2} \frac{\left(\Yi{i} - \sqrt{\bar{\alpha}_{i}}\Yi{0}\right)^\top  \left(\Yi{i} - \sqrt{\bar{\alpha}_{i}}\Yi{0}\right)}{1-\bar{\alpha_i}})$. Its log-likelihood gradient is $\nabla_{\Yi{i}}\pYi{i\mid 0 }(\Yi{i} \mid \Yi{0}) = - \frac{1}{1-\bar{\alpha}_i} (\Yi{i} - \sqrt{\bar{\alpha}_{i}}\Yi{0}) p_{i\mid 0}(\Yi{i} \mid \Yi{0})$.
Given $\Yi{0}$ as the integration variable in~\cref{eq:score_d}, $\pYi{i\mid 0 }(\Yi{i} \mid \Yi{0})$ is evaluated as a function of $\Yi{0}$ parameterized by $\Yi{i}$. Based on that, we define the function $\phi_i(\Yi{0})$ as:
\begin{align}
    \phi_i(\Yi{0}) \propto \pYi{i\mid 0 }(\Yi{i} \mid \Yi{0})  \propto \exp(-\frac{1}{2} \frac{\left(\Yi{0}-\frac{\Yi{i}}{\sqrt{\bar{\alpha}_{i}}}\right)^\top  \left(\Yi{0}-\frac{\Yi{i}}{\sqrt{\bar{\alpha}_{i}}}\right)}{\frac{1-\bar{\alpha}_i}{\bar{\alpha}_{i}}}) \propto \mathcal{N}(\frac{\Yi{i}}{\sba{i}}, \frac{I}{\bar{\alpha}_{i}} - I)
    \label{eq:prop_phi}
\end{align}
This finding enables the Monte-Carlo estimation for computing the score function. We collect a batch of samples from $\phi_i(\cdot)$ which we denote as $\mathcal{Y}^{(i)}$ and approximate the score as:
\begin{subequations}
    \begin{align}
                & \nabla_{\Yi{i}}\log{\pYi{i}(\Yi{i})}    = -\frac{\Yi{i}}{1-\bar{\alpha}_{i}} + \frac{\sqrt{\bar{\alpha}_{i}}}{1-\bar{\alpha}_{i}} \frac{\int \Yi{0} \phi_i(\Yi{0}) \pYi{0}(\Yi{0}) \, d\Yi{0}}{\int  \phi_i(\Yi{0})  \pYi{0}(\Yi{0}) \, d\Yi{0}}                                                              \label[equation]{eq:score} \\
        \approx & -\frac{\Yi{i}}{1-\bar{\alpha}_{i}} + \frac{\sqrt{\bar{\alpha}_{i}}}{1-\bar{\alpha}_{i}} \underbrace{\frac{\sum_{\Yi{0}\in\mathcal{Y}^{(i)}} \Yi{0} \pYi{0}(\Yi{0})}{\sum_{\Yi{0}\in \mathcal{Y}^{(i)}} \pYi{0}(\Yi{0})}}_{\text{Monte Carlo Approximation}} \label[equation]{eq:Y0_bar}
        :=                                         -\frac{\Yi{i}}{1-\bar{\alpha}_{i}} + \frac{\sqrt{\bar{\alpha}_{i}}}{1-\bar{\alpha}_{i}}\bar{Y}^{(0)}(\mathcal{Y}^{(i)})                                                                                                                                                                 %
    \end{align}
\end{subequations}

\begin{table}[h!]
    \begin{small}
        \centering
        \begin{tabular}{@{}p{0.20\linewidth}p{0.44\linewidth}p{0.3\linewidth}@{}}
            \toprule
            \textbf{Aspect}     & \textbf{Model-Based Diffusion (MBD)}                                                 & \textbf{Model-Free Diffusion (MFD)}          \\
            \midrule
            Target Distribution & Known (\cref{eq:p0}), but hard to sample                                             & Unknown, but have data                       \\
            Objective           & Sample $\Yi{0}$ from high-likelihood region of $\pYi{0}(\cdot)$                      & Sample $\Yi{0} \sim \pYi{0}(\cdot)$          \\
            Score Approximation & Estimated using the model (\cref{eq:score}). Can be augmented with demonstrations (\cref{eq:target_demo,eq:Y0_bar_demo}) & Learned from data                            \\
            Backward Process    & Perform \SA{} (\cref{eq:reverse}) to move samples towards most-likely states         & Run reverse SDE to preserve sample diversity \\
            \bottomrule
        \end{tabular}
        \caption{Comparison of Model-Based Diffusion (MBD) and Model-Free Diffusion (MFD)}
        \label{tab:mbd_mfd}
    \end{small}
\end{table}

\textbf{Comparison between MFD and MBD. }\cref{tab:mbd_mfd} highlights the key differences between MBD and MFD, which originate from two assumptions made in MBD:
(a) a known target distribution $\pYi{0}(\Yi{0})$ given the model;
(b) the objective of sampling $\Yi{0}$ from the high-likelihood region of $\pYi{0}(\Yi{0})$ to minimize the cost function.
For (a), MBD leverages $\pYi{0}$ to estimate the score following~\cref{eq:score}, whereas MFD learns that from the data.
For (b), MBD runs \SA{} in \cref{eq:reverse} to quickly move the samples to the high-density region as depicted in~\cref{fig:distributions}(c), while MFD runs reverse SDE $\sYi{i-1}{k} = \frac{1}{\sqrt{\alpha_i}}\left(\sYi{i}{k}+\frac{1-\alpha_i}{2} \nabla_{\Yi{i}}\log \pYi{i}(\sYi{i}{k})\right)+\sqrt{1-\alpha_i} \mathbf{z}_i$, where $\mathbf{z}_i$ is Gaussian noise, to maintain the sample diversity.
Given low temperature $\lambda$, it can be shown that $\nabla\log{\pYi{i}(\Yi{i})} \approx -\frac{1}{(1-\bar{\alpha}_{i})}(\Yi{i} - \argmax{\pYi{i}(\cdot)})$ %
\footnote{See more elaborations in \cref{sec:Apd_low_t}}, 
i.e., the function $\log{\pYi{i}(\Yi{i})}$ is $\frac{1}{(1-\bar{\alpha}_{i})}$-smooth. Therefore, choosing the step size $(1-\bar{\alpha}_{i})$ in~\cref{eq:reverse} is considered optimal, as for $L$-smooth functions, $O(\frac{1}{L})$ is the step size that achieves the fastest convergence \cite{zeilerADADELTAAdaptiveLearning2012}.

\textbf{How diffusion helps?} The diffusion process plays an important role in helping \SA{} overcome the local minimum issue in highly non-convex optimization problems, as shown in ~\cref{fig:distributions}(a).
Compared with optimizing a highly non-convex objective, \SA{} is applied to the intermediate distribution $\pYi{i}(\cdot) = \int \pYi{0}(\Yi{0})p_{i\mid 0}(\cdot) d\Yi{0}$, which is made concave by convoluting $\pYi{0}(\cdot)$ with a Gaussian distribution $\pYi{i \mid 0}(\cdot)$, as shown in~\cref{fig:distributions}(b).
Starting from the strongly concave Gaussian distribution $\pYi{N} \sim \mathcal{N}(\mathbf{0}, I)$ with scale $\bar{\alpha}_N \rightarrow 0$, the density is easy to sample. The covariance of the sampling density $\Sigma_{\phi_i}=(\frac{1}{\bar{\alpha}_{i}} - 1) I$ is large when $i=N$, implying that we are searching a wide space for global minima.
In the less-noised stage, the intermediate distribution $\pYi{i}(\cdot)$ is more peaked and closer to the target distribution $\pYi{0}(\cdot)$, and $\bar{\alpha}_i \rightarrow 1$ produces a smaller sampling covariance $\Sigma_{\phi_i}$ to perform a local search.
By iteratively running gradient ascent on the smoothed distribution, MBD can effectively optimize a highly non-convex objective function as presented in~\cref{fig:distributions}. The MBD algorithm is formally depicted in~\cref{alg:mbd}.

\begin{algorithm}
    \caption{Model-based Diffusion for Generic Optimization}
    {\small
        \begin{algorithmic}[1]
            \State \textbf{Input:} $\Yi{N} \sim \mathcal{N}(\mathbf{0}, I)$
            \For {$i=N$ to $1$}
            \State Sample $\mathcal{Y}^{(i)} \sim \mathcal{N}(\frac{\Yi{i}}{\sba{i-1}}, (\frac{1}{\bar{\alpha}_{i-1}}-1) I)$ \label{alg:sample}
            \State Calculate~\cref{eq:Y0_bar} $\bar{Y}^{(0)}(\mathcal{Y}^{(i)})=\frac{\sum_{\Yi{0}\in\mathcal{Y}^{(i)}} \Yi{0} \pYi{0}(\Yi{0})}{\sum_{\Yi{0}\in \mathcal{Y}^{(i)}} \pYi{0}(\Yi{0})}$ \label{alg:mean}
            \State Estimate the score~\cref{eq:score}: $\nabla_{\Yi{i}}\log{\pYi{i}(\Yi{i})}
                \approx -\frac{\Yi{i}}{1-\bar{\alpha}_{i}} + \frac{\sqrt{\bar{\alpha}_{i}}}{1-\bar{\alpha}_{i}}\bar{Y}^{(0)}(\mathcal{Y}^{(i)})$ \label{alg:score}
            \State \SA{}~\cref{eq:reverse}: $\Yi{i-1} \leftarrow  \frac{1}{\sqrt{\alpha_{i}}}\left(\Yi{i} + (1-\bar{\alpha}_{i})\nabla_{\Yi{i}}\log{\pYi{i}(\Yi{i})}\right)$ \label{alg:reverse}
            \EndFor
        \end{algorithmic}
        \label[algorithm]{alg:mbd}
    }
\end{algorithm}

\textbf{Connection with Sampling-based Optimization. }
When diffusion step is set to $N=1$, MBD effectively reduces to the Cross-Entropy Method (CEM)~\cite{botevChapterCrossEntropyMethod2013} for optimization.
To see this, we can plug the estimated score~\cref{eq:Y0_bar} into the \SA{}~\cref{eq:reverse} and set $N=1$: $\Yi{0} = \frac{\bar{\alpha}_1}{\alpha_1} \bar{Y}^{(0)}(\mathcal{Y}^{(1)}) = \bar{Y}^{(0)}(\mathcal{Y}^{(1)})=\frac{\sum_{\Yi{0} \in \mathcal{Y}^{(1)}} \Yi{0} w(\Yi{0})}{\sum_{\Yi{0} \in \mathcal{Y}^{(1)}} w(\Yi{0})}$ where $w(\Yi{0}) = p_0(\Yi{0}) \propto \exp(-\frac{J(\Yi{0})}{\lambda})$ and $\mathcal{Y}^{(1)} \sim \mathcal{N}(\frac{\Yi{1}}{\alpha_0}, (\frac{1}{\alpha_{0}}-1) I)$.
This precisely mirrors the update mechanism in CEM, which aims to optimize the objective function $f_\text{CEM}(\Yi{0})=J(\Yi{0})$ and determine the sampling covariance $\Sigma_\text{CEM}=(\frac{1}{\alpha_{0}} - 1) I$, thus linking the sampling strategy of CEM with the $\alpha$ schedule in MBD.
The advances that distinguish MBD from CEM-like methods are (1) the careful scheduling of $\alpha$ and (2) the intermediate refinements on $p_i$, both following the forward diffusion process. This allows MBD to optimize for smoothed functions in the early stage and gradually refine the solution for the original objective. On the contrary, CEM's solution could either be biased given a large $\Sigma_\text{CEM}$ which overly smoothes the distribution as in $\pYi{20},\pYi{100}$ of~\cref{fig:distributions}(b), or stuck in local minima with a small $\Sigma_\text{CEM}$ as in $\pYi{1}$ of~\cref{fig:distributions}(b) where the distribution is highly non-concave.

\subsection{Model-based Diffusion for Trajectory Optimization}
\label[section]{sec:algo_to}

For TO, we have to accommodate the constraints in~\cref{eq:to_problem} which change the target distribution to $\pYi{0}(\Yi{0}) \propto p_d(\Yi{0})p_J(\Yi{0})p_g(\Yi{0})$.
Given that $p_d(\Yi{0})$ is a Dirac delta function that assigns non-zero probability only to dynamically feasible trajectories, sampling from $\phi_i(\Yi{0})$ could result in low densities.
To enhance sampling efficiency, we collect a batch of dynamically feasible samples $\mathcal{Y}^{(i)}_d $ from the distribution  $\phi_i(\Yi{0})p_d(\Yi{0})$ with model information. Proceeding from~\cref{eq:score}, and incorporating $\pYi{0}(\Yi{0}) \propto p_d(\Yi{0})p_J(\Yi{0})p_g(\Yi{0})$, we show the score function is:

\vspace{-0.5cm}
\begin{subequations}
    {\small
        \begin{align}
            \nabla_{\Yi{i}}\log{\pYi{i}(\Yi{i})}     
            =                                         & -\frac{\Yi{i}}{1-\bar{\alpha}_{i}} + \frac{\sqrt{\bar{\alpha}_{i}}}{1-\bar{\alpha}_{i}} \frac{\int \Yi{0} \phi_i(\Yi{0}) p_d(\Yi{0})p_g(\Yi{0})p_J(\Yi{0}) \, d\Yi{0}}{\int  \phi_i(\Yi{0})  p_d(\Yi{0})p_g(\Yi{0})p_J(\Yi{0}) \, d\Yi{0}}                       \\
            \approx                                   & -\frac{\Yi{i}}{1-\bar{\alpha}_{i}} + \frac{\sqrt{\bar{\alpha}_{i}}}{1-\bar{\alpha}_{i}} \frac{\sum_{\Yi{0} \in \mathcal{Y}^{(i)}_d} \Yi{0} p_J(\Yi{0}) p_g(\Yi{0})}{\sum_{\Yi{0}\in \mathcal{Y}^{(i)}_d} p_J(\Yi{0}) p_g(\Yi{0})} \label[equation]{eq:Y0_bar_to} \\
            =                                         & -\frac{\Yi{i}}{1-\bar{\alpha}_{i}} + \frac{\sqrt{\bar{\alpha}_{i}}}{1-\bar{\alpha}_{i}}\bar{Y}^{(0)},
            \label[equation]{eq:score_to}                                                                                                                                                                                                                                                                                \\
            \mathrm{where} \quad \bar{Y}^{(0)} =                           & \frac{\sum_{\Yi{0} \in \mathcal{Y}^{(i)}_d} \Yi{0} w(\Yi{0})}{\sum_{\Yi{0} \in \mathcal{Y}^{(i)}_d} w(\Yi{0})}, \quad w(\Yi{0}) = p_J(\Yi{0})p_g(\Yi{0}) \label[equation]{eq:Y0_bar_to2}
        \end{align}
    }
\end{subequations}

The model plays a crucial role in score esitimation by transforming infeasible samples $\mathcal{Y}^{(i)}$ from~\cref{alg:sample_to} in~\cref{alg:mbd_to} into feasible ones $\mathcal{Y}^{(i)}_d$.
The conversion is achieved by putting the control part $U=u_{1:T}$ of $\Yi{0} = [x_{1:T};u_{1:T}]$ into the dynamics~\cref{eq:to_dyn} recursively to get the dynamically feasible samples $\Yi{0}_{d}$ (\cref{alg:rollout}), which shares the same idea with the shooting method~\cite{howellALTROFastSolver2019} in TO.
MBD then evaluates the weight of each sample with $p_g(\Yi{0})p_J(\Yi{0})$ in~\cref{alg:mean_to}.
One common limitation of shooting methods is that they could be inefficient for long-horizon tasks due to the combinatorial explosion of the constrained space $p_g(Y) \propto \prod_{t=1}^{T} \mathbf{1}(g_t(x_t, u_t)\leq 0)$, leading to low constraint satisfaction rates.
To address this issue, we will introduce demonstration-augmented MBD in~\cref{sec:algo_demo} to guide the sampling process from the state space to improve sample quality.

\begin{algorithm}
    \caption{Model-based Diffusion for Trajectory Optimization}
    {\small
        \begin{algorithmic}[1]
            \State \textbf{Input:} $\Yi{N} \sim \mathcal{N}(\mathbf{0}, I)$
            \For {$i=N$ to $1$}
            \State Sample $\mathcal{Y}^{(i)} \sim \mathcal{N}(\frac{\Yi{i}}{\sba{i-1}}, (\frac{1}{\bar{\alpha}_{i-1}}-1) I)$ \label{alg:sample_to}
            \State Get dynamically feasible samples: $\mathcal{Y}^{(i)}_d \leftarrow \text{rollout}(\mathcal{Y}^{(i)})$ \label{alg:rollout}
            \State Calculate $\bar{Y}^{(0)}$ with~\cref{eq:Y0_bar_to2} (model only) or \cref{eq:Y0_bar_demo} (model $+$ demonstration) \label{alg:mean_to}
            \State Estimate the score~\cref{eq:score_to}: $\nabla_{\Yi{i}}\log{\pYi{i}(\Yi{i})}
                \approx -\frac{\Yi{i}}{1-\bar{\alpha}_{i}} + \frac{\sqrt{\bar{\alpha}_{i}}}{1-\bar{\alpha}_{i}}\bar{Y}^{(0)}$ \label{alg:score_to}
            \State \SA{}~\cref{eq:reverse}: $\Yi{i-1} \leftarrow  \frac{1}{\sqrt{\alpha_{i}}}\left(\Yi{i} + (1-\bar{\alpha}_{i})\nabla_{\Yi{i}}\log{\pYi{i}(\Yi{i})}\right)$ \label{alg:reverse_to}
            \EndFor
        \end{algorithmic}
    }
    \label{alg:mbd_to}
\end{algorithm}

\subsection{Model-based Diffusion with Demonstration}
\label[section]{sec:algo_demo}

With the ability to leverage model information, MBD can also be seamlessly integrated with various types of data, including imperfect or partial-state demonstrations by modeling them as noisy observations of the desired trajectory $p(Y_\text{demo} \mid \Yi{0}) \sim \mathcal{N}(\Yi{0}, \sigma^2 I)$.
Given suboptimal demonstrations, sampling from the posterior $p(\Yi{0} \mid Y_\text{demo}) \propto p_0(\Yi{0}) p(Y_\text{demo} \mid \Yi{0})$ could lead to poor solutions as the demonstration likelihood $p(Y_\text{demo} \mid \Yi{0})$ could dominate the model-based distribution $p_0(\Yi{0}) \propto p_d(\Yi{0}) p_J(\Yi{0}) p_g(\Yi{0})$ and mislead the sampling process. Rather, we assess $\Yi{0}$ using $p(Y_\text{demo} \mid \Yi{0})$, employing a similar technique to interchange the distribution's parameter with the random variable, as demonstrated in~\cref{eq:prop_phi}, to establish $p_{\text{demo}}(\Yi{0}) \propto p(Y_\text{demo} \mid \Yi{0}) \propto \mathcal{N}(\Yi{0}\mid Y_\text{demo}, \sigma^2 I)$.

To accommodate demonstrations of varying qualities, instead of fixing target to $p_0(\Yi{0}) p(Y_\text{demo} \mid \Yi{0})$, we propose seperating the $\pYi{0}(\Yi{0})$ from $p_\text{demo}(\Yi{0})$ to form a new target distribution\footnote{A comparison between the demonstration-augmented MBD and the vanilla MBD is illustrated in~\cref{fig:distributions_demo} with detailed breakdowns in~\cref{sec:demo_exp}.}:
\begin{equation}
    p'_0(\Yi{0}) \propto (1-\eta) p_d(\Yi{0})p_J(\Yi{0})p_g(\Yi{0}) + \eta p_\text{demo}(\Yi{0})p_J(Y_\text{demo})p_g(Y_\text{demo}) \label{eq:target_demo}
\end{equation}
where $\eta$ is a constant to balance the model and the demonstration.
Here, we have introduced two extra constant terms $p_J(Y_\text{demo})$ and $p_g(Y_\text{demo})$ to ensure that the demonstration likelihood is properly scaled to match the model likelihood $p_0(\Yi{0})$.
With these preparations, we propose to adaptively determine the significance of the demonstration by choosing $\eta$ as follows:
\begin{equation}
    \eta = \begin{cases}
        1 & p_d(\Yi{0})p_J(\Yi{0})p_g(\Yi{0}) < p_\text{demo}(\Yi{0})p_J(Y_\text{demo})p_g(Y_\text{demo}) \\
        0 & p_d(\Yi{0})p_J(\Yi{0})p_g(\Yi{0}) \ge p_\text{demo}(\Yi{0})p_J(Y_\text{demo})p_g(Y_\text{demo}).
    \end{cases}
\end{equation}
When samples have a high model-likelihood $p_0$, we ignore the demonstration and sample from the model. Otherwise, we trust the demonstration.
With the demonstration-augmented target distribution, we modify the way to calculate %
$\bar{Y}^{(0)}$ in \cref{eq:Y0_bar_to2} as follows to obtain the score estimate:
\begin{align}
    \bar{Y}^{(0)} =                           \frac{\sum_{\Yi{0}\in \mathcal{Y}^{(i)}_d} \Yi{0} w(\Yi{0})}{\sum_{\Yi{0}\in \mathcal{Y}^{(i)}_d} w(\Yi{0})}, \quad w(\Yi{0}) = \max{\left\{
        \begin{array}{c}
            p_d(\Yi{0})p_J(\Yi{0})p_g(\Yi{0}), \\
            p_\text{demo}(\Yi{0})p_J(Y_\text{demo})p_g(Y_\text{demo})
        \end{array}
        \right\}}. \label[equation]{eq:Y0_bar_demo}
\end{align}

\section{Experimental Results}

\label[section]{sec:exp}
The experimental section will focus on demonstrating the capabilities of MBD in: (1) its effectiveness as a zeroth-order solver for high-dimensional, non-convex, and non-smooth trajectory optimization problems, and (2) its flexibility in utilizing dynamically infeasible data to enhance performance and regularize solutions.
Our benchmark shows that MBD outperforms PPO by $59\%$ in various control tasks with $10\%$ computational time.

Beyond control problems, in Appendix \ref{sec:Apd_BBO}, we also show that MBD significantly improves sampling efficiency by an average of $23\%$ over leading baselines in high-dimensional (up to 800d) black-box optimization testbeds ~\cite{hansenCMAEvolutionStrategy2023,erikssonScalableGlobalOptimization2019,wangLearningSearchSpace2020,nayebiFrameworkBayesianOptimization2019,liuVersatileBlackBoxOptimization2020,papenmeierIncreasingScopeYou2023}. We also apply MBD to optimize an MLP network with 28K parameters in a \emph{gradient-free} manner, achieving 86\% accuracy within 2s for the MNIST classification task~\cite{MNISTDatabaseHandwritten}, which is comparable to the gradient-based optimizer (SGD with momentum, 93\% accuracy). 

\begin{table}[ht]
    \begin{small}
        \centering
        \begin{tabular}{lccccc}
            \hline
            \textbf{Task}    & \textbf{CMA-ES}  & \textbf{CEM}     & \textbf{MPPI}    & \textbf{RL}      & \textbf{MBD}                   \\
            \hline
            \hline
            Hopper           & $1.12 \pm 0.10$  & $0.65 \pm 0.12$  & $0.91 \pm 0.15$  & $1.40 \pm 0.04$  & $\mathbf{1.53} \pm \mathbf{0.03}$ \\
            Half Cheetah     & $0.44 \pm 0.10$  & $0.22 \pm 0.15$  & $0.20 \pm 0.14$  & $1.59 \pm 0.05$  & $\mathbf{2.31} \pm \mathbf{0.19}$ \\
            Ant              & $1.18 \pm 0.52$  & $0.85 \pm 0.17$  & $0.33 \pm 0.45$  & $3.26 \pm 1.61$  & $\mathbf{3.80} \pm \mathbf{0.35}$ \\
            Walker2D         & $0.83 \pm 0.04 $ & $1.06 \pm 0.04 $ & $0.90 \pm 0.05$  & $1.09 \pm 0.28$  & $\mathbf{2.63} \pm \mathbf{0.23}$ \\
            Humanoid Standup & $0.58 \pm 0.01 $ & $0.47 \pm 0.01$  & $0.53 \pm 0.05$  & $0.83 \pm 0.02$  & $\mathbf{0.99} \pm \mathbf{0.07}$ \\
            Humanoid Running & $0.60 \pm 0.11$  & $0.41 \pm 0.16$  & $0.59 \pm 0.14$  & $1.80 \pm 0.03$  & $\mathbf{2.92} \pm \mathbf{0.26}$ \\
            Push T           & $0.39 \pm 0.07$  & $0.25 \pm 0.09$  & $-0.13 \pm 0.09$ & $-0.63 \pm 0.16$ & $\mathbf{0.67} \pm \mathbf{0.10}$ \\
            \hline
        \end{tabular}
        \caption{Reward of different methods on non-continuous tasks.}
        \label{tab:reward}
    \end{small}
\end{table}

\begin{table}[ht]
    \begin{small}
        \centering
        \begin{tabular}{lccccc}
            \hline
            \textbf{Task}     & \textbf{CMA-ES} & \textbf{CEM}    & \textbf{MPPI}   & \textbf{RL}      & \textbf{MBD} \\
            \hline
            \hline
            Hopper            & \osec{29.3}     & \osec{26.5}     & \osec{26.4}     & \msec{17}{45.63} & \osec{26.5}     \\
            Half Cheetah      & \osec{29.5}     & \osec{26.4}     & \osec{26.7}     & \msec{4}{18.8}   & \osec{26.8}     \\
            Ant               & \osec{18.4}     & \osec{16.1}     & \osec{16.0}     & \msec{2}{46.8}   & \osec{16.2}     \\
            Walker2D          & \osec{37.5}     & \osec{34.5}     & \osec{34.7}     & \msec{5}{1.5}    & \osec{34.6}     \\
            Humanoid Standup  & \osec{20.8}     & \osec{17.6}     & \osec{17.7}     & \msec{4}{29}     & \osec{17.7}     \\
            Humanoid  Running & \osec{30.8}     & \osec{29.7}     & \osec{29.6}     & \msec{3}{34.7}   & \osec{30.0}     \\
            Push T            & \msec{10}{40.0} & \msec{10}{32.0} & \msec{10}{32.3} & \msec{67}{25.6}  & \msec{10}{32.8} \\
            \hline
        \end{tabular}
        \caption{Computational time of different methods on non-continuous tasks.}
        \label{tab:time}
    \end{small}
\end{table}

\subsection{MBD for Planning in Contact-rich Tasks}
\label{sec:exp_noncontinuous}

To test the effectiveness of MBD as a trajectory optimizer for systems involving non-smooth dynamics, we run MBD on both locomotion and manipulation tasks detailed in \cref{sec:env}.
The locomotion tasks includes hopper, half-cheetah, ant, walker2d, humanoid-standup, and humanoid-running. %
The selected manipulation task, pushT~\cite{chiDiffusionPolicyVisuomotor2023}, presents its own challenges due to the complexity introduced by contact dynamics and the long-horizon nature of the task.
These tasks are widely considered difficult due to their hybrid nature and high dimensionality.

MBD is compared with the state-of-the-art zeroth-order optimization methods, including CMA-ES~\cite{akimotoTheoreticalFoundationCMAES2012}, CEM~\cite{botevChapterCrossEntropyMethod2013}, and MPPI~\cite{williamsInformationTheoreticModelPredictive2018}, as well as reinforcement learning (RL) algorithms (e.g., PPO~\cite{schulmanProximalPolicyOptimization2017} and SAC~\cite{haarnoja2018soft}) on these tasks.
Please note that we use MBD and zeroth-order baselines to generate control sequences and replay them in an open-loop manner, whereas RL generates a closed-loop policy.
The RL implementation follows the high-performance parallelized framework from Google Brax~\cite{freemanBraxDifferentiablePhysics2021} elaborated in~\cref{sec:baseline}.
For the zeroth-order optimizer, we match the iteration and sample number with the MBD.
All the experiments were conducted on a single NVIDIA RTX 4070 Ti GPU.
Quantitative metrics including the average step reward and the computational time tested over $50$ steps repeated for $8$ seeds are reported in \cref{tab:reward,tab:time}.
MBD substantially outperforms zeroth-order optimization methods and even outperforms RL in most tasks. Specifically, for the pushT task, MBD achieves a significantly higher reward than the RL algorithm thanks to its iterative refinement process, which effectively explores the full control space while keeping fine-grained control to precisely push the object.
Compared with the computationally heavy RL algorithms, MBD only requires one-tenth of time, which is similar to other zeroth-order optimization methods. The optimization process of MBD is visualized in \cref{fig:traj}, where the iterative refinement process with the model plays a crucial role in optimizing high-dimensional tasks.

\begin{figure}[ht]
    \centering
    \includegraphics[width=1.0\linewidth]{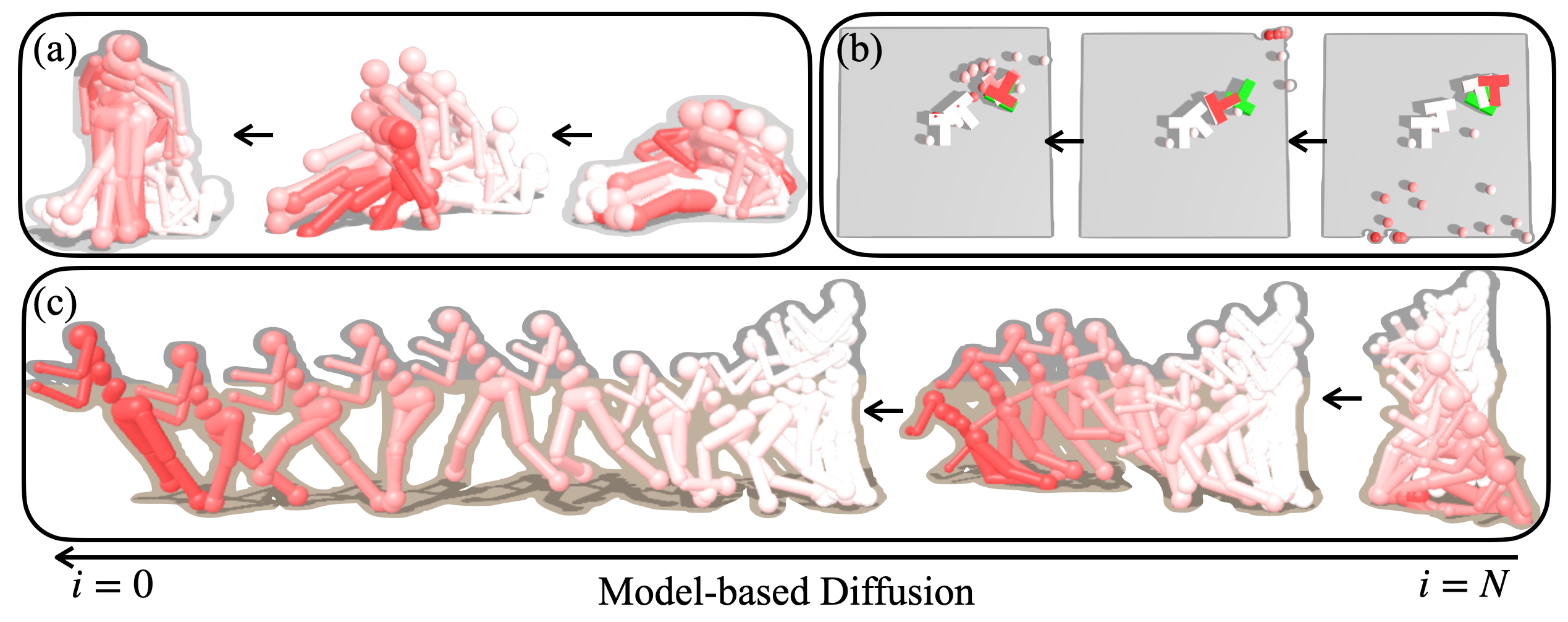}
    \vspace{-0.4cm}
    \caption{Optimization process of MBD on the (a) Humanoid Standup, (b) Push T, and (c) Humanoid Running tasks. The trajectory is iteratively refined to achieve the desired objective in the high-dimensional space with model information.}
    \label{fig:traj}
    \vspace{-0.5cm}
\end{figure}

\begin{figure}[ht]
    \centering
    \includegraphics[width=1.0\linewidth, bb=0 0 2356 518]{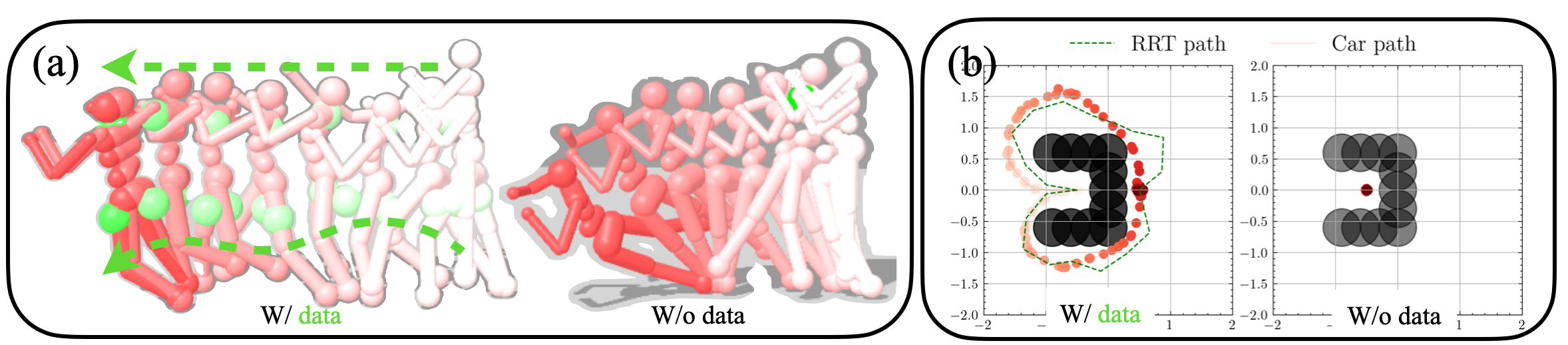}
    \caption{MBD optimized trajectory with data augmentation on the (a) Humanoid Jogging and (b) Car UMaze Navigation tasks. With \textcolor{green}{data augmentation}, the trajectory is regularized and refined to achieve the desired objective.}
    \label{fig:traj_demo}
\end{figure}

\subsection{Data-augmented MBD for Trajectory Optimization}
We also evaluate the performance of MBD with data augmentation on the Car UMaze Navigation and Humanoid Jogging tasks to see how partial and dynamically infeasible data can help the exploration of MBD and regularize the solution by steering the diffusion process.%

For Car UMaze Navigation, %
the map blocked by U-shaped obstacles is challenging to explore given a nonlinear dynamics model. Therefore, random shooting has a low chance of reaching the goal region. %
To sample with loosened dynamical constraints, we augment MBD with data from the RRT~\cite{lavalleRapidlyexploringRandomTrees1998} algorithm through goal-directed exploration with simplified dynamics.
\cref{fig:traj_demo}(b) shows the difference between data-augmented MBD and data-free one: the former can refine the infeasible trajectory and further improve it to reach the goal in less time, while the latter struggles to find a feasible solution.
The reason is that the infeasible trajectory from RRT serves as a good initialization for MBD, which can be further refined to minimize the cost function with MBD.

For Humanoid Jogging, we aim to regularize the solution for the task with multiple solutions to the desired one with partial state data.
Due to the infinite possibilities for humanoid jogging motion, the human motion data provide a good reference to regularize MBD to converge to a more human-like and robust solution instead of an aggressive or unstable one~\cite{heLearningHumantoHumanoidRealTime2024,pengAMPAdversarialMotion2021}.
We use data from the CMU Mocap dataset, from which we extract torso, thigh, and shin positions and use them as a partial state reference.
\cref{fig:traj_demo}(a) demonstrates a more stable motion generated by data-augmented MBD.

\section{Conclusion and Future Work}

\label{sec:conclusion}

This paper introduces Model-Based Diffusion (MBD), a novel diffusion-based trajectory optimization framework that employs a dynamics model to approximate the score function.
MBD not only outperforms existing methods in terms of sample efficiency and generalization, but also provides a new perspective on trajectory optimization by leveraging diffusion models as powerful samplers. 
Future directions involve theoretically understanding its convergence, optimizing the standard Gaussian forward process using the model information, adapting it to online tasks with receding horizon strategies, and exploring advanced sampling and scheduling techniques to further improve performance.

\bibliographystyle{plain}
\bibliography{ref}

\appendix

\section{Appendix / Supplemental Material}

\subsection{Convergence of Distribution with Small Temperature}
\label{sec:Apd_low_t}
We first give the definition of the volume of the sub-level set for cost $J$.
\begin{defin}
    Let \( F : \mathcal{R}^d \rightarrow \mathcal{R} \) be a measurable function. Define the volume of the sub-level set for a given level \( t \) as: %
    \[
    V_F(t) = \int_{\mathcal{R}^d} \chi_{\{Y \in \mathcal{R}^d : F(Y) \leq t\}}(Y) \, dY,
    \]
    where \( \chi \) denotes the indicator function.
    \label[definition]{def:1}
\end{defin}
The volume function \( V_J(t) \)  plays a crucial role in linking geometric properties with probabilistic outcomes in optimization and learning algorithms. This function provides a quantitative measure that helps us to understand how changes in parameters like \( \lambda \) influence the distribution and concentration of probability mass. 

The interplay between geometry and probability, represented by \( V_J(t) \), is crucial for evaluating the convergence and stability of algorithms. It provides a significant method for utilizing the PDF of the random variable $\mathbf{Y}$ to constrain the CDF, thereby facilitating convergence in distribution.

\begin{prop}
    Given the target distribution $ \boldsymbol{Y} \sim p(\cdot)$ with $P(Y) \propto \exp\left(-\frac{J(Y)}{\lambda}\right), Y \in \mathcal{R}^d $, where \( J \) is a cost function with \(\min_Y J(Y) = 0\) and \( Y^* = \argmin J(Y) \), and assuming that the volume function \( V_J(t) \) is bounded by polynomial inequalities:
    \[
    Poly_l(t) \leq V_J(t)\leq Poly_u(t),
    \]
    where \( Poly_l(t) = \sum_{k=0}^M c^l_k t^{\alpha_k} \) and \( Poly_u(t) = \sum_{k=0}^M c^u_k t^{\alpha_k} \) are polynomials with coefficients satisfying \( c^l_k = 0 \) if and only if \( c^u_k = 0 \). The exponent term satisfies that $\alpha_k \in \mathcal{R}$, and $0< \alpha_0<\alpha_1<\cdots<\alpha_M < \infty$, It follows that:
    \[
    \lim_{\lambda \rightarrow 0} J(\boldsymbol{Y}) \xrightarrow{p} J(Y^*) = 0.
    \]
    The cost value $J(Y)$ converges in probability to $J(Y^*)$ as $\lambda \rightarrow 0$.
    \label[proposition]{prop:1}
\end{prop}
The condition on the polynomial bounds of $V_J(t)$ is generally not restrictive. For instance, consider $J = \eta_c \left\| Y - Y_* \right\|^m$, where $Y^*$ is the optimal point and $\eta_c > 0$ is any constant. In this case, $V_J(t) = C t^{\frac{d}{m}}$, where $C$ is a constant, meets the constraint in a straightforward way. This condition can be extended beyond this simple scenario, as even if $J$ has multiple modes, it can still adhere to this polynomial constraint.

\begin{proof}
  The convergence in distribution of $\boldsymbol{Y}$ towards $Y^*$ as $\lambda \rightarrow 0$ is established by analyzing the behavior of the probability density function, defined up to a multiplicative constant. Consider the density $\boldsymbol{Y}_0 \sim p_{\lambda}(Y)$ approximating $Y^*$ when $\lambda$ approaches zero. %
\begin{subequations}
    \begin{align}
        P(J(\mathbf{Y})\leq t) & = \int_{\{J(Y) \leq t\}} p(Y) dY ,\label[equation]{eq:sub_int1}\\
& = \int_0^t \int_{\{J(Y) = x\}}p(Y) dY dx, \label[equation]{eq:sub_int2}\\
& \propto \int_0^t exp(-\frac{x}{\lambda})\int_{\{J(Y) = x\}} dY dx ,\label[equation]{eq:sub_int3}\\
& = \int_0^t exp(-\frac{x}{\lambda}) \frac{d V_J(x)}{dx} dx ,\label[equation]{eq:sub_int4}
    \end{align}
\end{subequations}
where \cref{eq:sub_int3} is valid since $P(Y) \propto \exp \left(-\frac{J(Y)}{\lambda}\right)$ and $J(Y)$ represents the sufficient statistics of the distribution. We can obtain \cref{eq:sub_int4}  by computing the derivative of $V_J(x)$ based on the volume definition as shown in \cref{def:1}.

We denote $J_{\text{min}} = 	\min_Y J(Y)=0$ and $J_{\text{max}} = 	\max_Y J(Y)$ with $J_{\text{max}}$ satisfying $0 \leq J_{\text{max}} \leq+\infty$. We proceed to analyze \cref{eq:sub_int4} by performing integration by parts as shown in \cref{eq:int_byparts}.
\begin{subequations}
    \begin{align}
        \int_{J_{\text{min}}}^{J_{\text{max}}} \exp\left(-\frac{t}{\lambda}\right) dV_J(t)
        &= \int_{J_{\text{min}}}^{J_{\text{max}}}  d \left[\exp\left(-\frac{t}{\lambda}\right) V_J(t)\right] + \frac{1}{\lambda} \exp\left(-\frac{t}{\lambda}\right) V_J(t)dt \label[equation]{eq:int_byparts}\\
        &= \exp\left(-\frac{t}{\lambda}\right) V_J(t)\bigg|_{J_{\text{min}}}^{J_{\text{max}}} + \frac{1}{\lambda} \int_{J_{\text{min}}}^{J_{\text{max}}} \exp\left(-\frac{t}{\lambda}\right) V_J(t)dt.\label[equation]{eq:int_dense}
    \end{align}
\end{subequations}
    
To establish convergence in probability, we need to demonstrate that for any small $\epsilon > 0$ and $\delta > 0$, there exists sufficiently small $\lambda>0$, such that %
\begin{equation}
    P(J(\boldsymbol{Y}) < \epsilon) = \frac{\int_{0}^{\epsilon} \exp\left(-\frac{t}{\lambda}\right)dV_J(t)}{\int_{0}^{J_{\text{max}}} \exp\left(-\frac{t}{\lambda}\right)dV_J(t)} \geq  1-\delta.
     \label[equation]{eq:conv_ratio}
\end{equation}
where the equality is due to \cref{eq:sub_int4}.
Setting $\delta' = \frac{1-\delta}{\delta}$, it suffices to show that:
\begin{equation}
\label[equation]{eq:org_ratio}
    \frac{\int_{0}^{\epsilon} \exp\left(-\frac{t}{\lambda}\right) dV_J(t)}{\int_{\epsilon}^{J_{\text{max}}} \exp\left(-\frac{t}{\lambda}\right) dV_J(t)} \geq \delta'.
\end{equation}

Assuming without loss of generality that $J_{\text{max}} = \infty$, becuase  $dV_J(t)\geq 0$, $exp(-\frac{t}{\lambda})>0$, we have: %
\begin{equation}
    \frac{\int_{0}^{\epsilon} \exp\left(-\frac{t}{\lambda}\right) dV_J(t)}{\int_{\epsilon}^{J_{\text{max}}} \exp\left(-\frac{t}{\lambda}\right) dV_J(t)} \geq \frac{\int_{0}^{\epsilon} \exp\left(-\frac{t}{\lambda}\right) dV_J(t)}{\int_{\epsilon}^{\infty} \exp\left(-\frac{t}{\lambda}\right) dV_J(t)}.
\end{equation}

This ratio as in \cref{eq:org_ratio} can be expanded using the integral bounds and the polynomial approximations for $V_J(t)$,then it suffices to show that 

\begin{equation}
    \frac{\int_{0}^{\epsilon} \exp\left(-\frac{t}{\lambda}\right) dV_J(t)}{\int_{\epsilon}^{\infty} \exp\left(-\frac{t}{\lambda}\right) dV_J(t)} \geq \delta'.
    \label[equation]{eq:int_ratio}
\end{equation}

By inserting \cref{eq:int_dense} into both the numerator and denominator on the LHS of \cref{eq:int_ratio}, we obtain
\begin{subequations}
    \begin{align}
           \frac{\int_{0}^{\epsilon} exp(-\frac{t}{\lambda})dV_J(t)}{ \int_{\epsilon}^{\infty} exp(-\frac{t}{\lambda})dV_J(t)} & = \frac{exp(-\frac{\epsilon}{\lambda})V(\epsilon) + \frac{1}{\lambda}\int_{0}^{\epsilon}  exp(-\frac{t}{\lambda})V_J(t)dt}{-exp(-\frac{\epsilon}{\lambda})V(\epsilon) + \frac{1}{\lambda}\int_{\epsilon}^{\infty}  exp(-\frac{t}{\lambda})V_J(t)dt} \\
         &  \geq \frac{\int_{0}^{\epsilon}  exp(-\frac{t}{\lambda})V_J(t)dt}{ \int_{\epsilon}^{\infty}  exp(-\frac{t}{\lambda})V_J(t)dt}\\
         &  \geq \frac{\int_{0}^{\epsilon}  exp(-\frac{t}{\lambda})\sum_{k=0}^M c^l_k t^{\alpha_k} dt}{ \int_{\epsilon}^{\infty}  exp(-\frac{t}{\lambda})\sum_{k=0}^M c^u_k t^{\alpha_k}  dt}\label[equation]{eq:poly_ratio}
    \end{align}
\end{subequations}
To bound the expression in \cref{eq:poly_ratio}, we first derive the following integrals by utilizing a change of variables \( x = \frac{t}{\lambda} \) %
, which simplifies the expressions:
\begin{subequations}
\begin{align}
   \int_{0}^{\epsilon} \exp\left(-\frac{t}{\lambda}\right) t^{\alpha_k} \, dt &= \lambda^{k+1} \int_{0}^{\frac{\epsilon}{\lambda}} \exp(-x) x^{\alpha_k} \, dx, \label{eq:transform_int1} \\
   \int_{\frac{\epsilon}{\lambda}}^{\infty} \exp\left(-\frac{t}{\lambda}\right) t^{\alpha_k} \, dt &= \lambda^{k+1} \int_{\frac{\epsilon}{\lambda}}^{\infty} \exp(-x) x^{\alpha_k} \, dx. \label{eq:transform_int2}
\end{align}
\end{subequations}

For these transformed integrals, we can observe that \(\int_0^\infty \exp(-x) x^{\alpha_k} \, dx = \Gamma(\alpha_k+1)\), the gamma function, which is well-defined for all non-negative \( \alpha_k \). Given that $\delta'$ is a function of $\delta$ , by applying the intermediate value theorem and definition of the limit of the integral, we can choose $\epsilon_k$ in such a way that:
\[
\int_0^{\epsilon_k} \exp(-x) x^{\alpha_k} \, dx \geq \frac{c_k \delta'}{1 + c_k \delta'} \Gamma(\alpha_k+1),
\]

where \( c_k = \frac{c_k^l}{c_k^u} \) denotes the ratio of coefficients in polynomial lower and upper bounds for \( V_J(t)\). By selecting \(\epsilon_{\text{max}} = \max{\epsilon_0,\epsilon_1,\cdots,\epsilon_M}\) to be the maximum of all such \(\epsilon_k\), ensuring coverage for all polynomial terms up to \( M \), we establish that:
\begin{equation}
    \frac{\int_{0}^{\epsilon_{\text{max}}} \exp\left(-\frac{t}{\lambda}\right) c^l_k t^{\alpha_k} \, dt}{\int_{\epsilon_{\text{max}}}^{\infty} \exp\left(-\frac{t}{\lambda}\right) c^u_k t^{\alpha_k} \, dt} \geq \delta', \quad \text{for all } k = 0, 1, \ldots, M.
\end{equation}

By ensuring that $\lambda \leq \frac{\epsilon}{\epsilon_{\text{max}}}$, we can conclude:

\begin{equation}
    \frac{\int_{0}^{\epsilon} \exp\left(-\frac{t}{\lambda}\right) \sum_{k=0}^M c^l_k t^{\alpha_k} \, dt}{\int_{\epsilon}^{\infty} \exp\left(-\frac{t}{\lambda}\right) \sum_{k=0}^M c^u_k t^{\alpha_k} \, dt} \geq \delta',
\end{equation}

Thus, the condition specified in \cref{eq:conv_ratio} is satisfied, validating that the distribution of \( \boldsymbol{Y} \) converges in distribution to \( Y^* \) as \( \lambda \) approaches zero.

\end{proof}
By adding another mild assumption regarding the landscape of $J$ near the global optimum, we can demonstrate the convergence of the random variable $\textbf{Y}$ itself, rather than the convergence of $J(\textbf{Y})$.
\begin{defin}
    We denote the minimum of the complementary set of neighborhood as:
    $$
    J^*_{\mathcal{B}}(\delta) = \min_{\left\|Y - Y^* \right\| > \delta} J(Y) - J(Y^*).
    $$
    \label[definition]{def:2}
\end{defin}
\begin{prop}
    Given the context and conditions specified in \cref{def:1,def:2} and \cref{prop:1}, and given that $J$ has only one golbal minimizer $Y^*$, i.e. 
    there exist small $\delta^*$, that for $ \delta \in (0,\delta^*]$, $J^*_{\mathcal{B}}(\delta)$ is strictly increasing, and $J^*_{\mathcal{B}}(\delta^*) < \infty$. 
    It follows that:
    \[
    \lim_{\lambda \rightarrow 0} \boldsymbol{Y} \xrightarrow{p} Y^*.
    \]
    The random variable $\mathbf{Y}$ converges in probability to $Y^*$ as $\lambda \rightarrow 0$.
    \label[proposition]{prop:2}
\end{prop}
\begin{proof}

In order to prove that 
    $\lim_{\lambda \rightarrow 0} \boldsymbol{Y} \xrightarrow{p} Y^*.$ We need to prove that for any sufficient small $\gamma > 0$ and $\delta > 0$, there exists small $\lambda>0$, such that 
    \begin{equation}
        P(\left\|Y - Y^* \right\|\leq \delta)\geq 1-\gamma
    \end{equation}
From \cref{def:2} and due to the strict increase of $J^*_{\mathcal{B}}(\delta)$, 
\begin{equation}
    \left\|Y - Y^*\right\| \leq \delta ,\quad \forall Y \in \left\{Y \in \mathcal{R}^d \mid J(Y) - J(Y^*) < J^*_{\mathcal{B}}(\delta) \right\}, 
    \label[equation]{eq:Y_to_JY}
\end{equation}
where $0 < \delta \leq \delta^* $. Because if $\left\|Y - Y^*\right\| > \delta  $, $J(Y) - J(Y^*) < J^*_{\mathcal{B}}(\delta) = \min_{\left\|Y - Y^* \right\| > \delta} J(Y) - J(Y^*)$ contradicts \cref{def:2}. 

Given that $\lim_{\lambda \rightarrow 0} J(\boldsymbol{Y}) \xrightarrow{p} J(Y^*).$ and any sufficient small $\epsilon,\gamma>0$.
\begin{equation}
    P(J(Y) - J(Y^*)\leq \epsilon) \geq 1-\gamma
\end{equation}
Therefore, $\exists \lambda>0$, such that
\begin{equation}
    P(J(Y) - J(Y^*)\leq J^*_{\mathcal{B}}(\delta)) \geq 1-\gamma.
\end{equation}
And From \cref{eq:Y_to_JY}, we have that
\begin{equation}
   P(\left\|Y - Y^*\right\| \leq \delta) \geq P(J(Y) - J(Y^*)\leq J^*_{\mathcal{B}}(\delta)) \geq 1-\gamma
\end{equation}
We have that  $\mathbf{Y}$ converges in probability to $Y^*$ ,i.e, $\lim_{\lambda \rightarrow 0} \boldsymbol{Y} \xrightarrow{p} Y^*.$
    
\end{proof}
\begin{prop}
    Given the context and conditions specified in \cref{prop:1,prop:2} and the way we define the forward process as in \cref{eq:forward}. The diffused $Y_i$ converge in density to a Gaussian distribution.
    $$
    \lim_{\lambda \rightarrow 0} \boldsymbol{\Yi{i}}\xrightarrow{d} \mathcal{N}(\sqrt{\bar{\alpha}_i} Y^*, \sqrt{1-\bar{\alpha}_i} I),
    $$
    where $\boldsymbol{\Yi{i}} \sim \pYi{i}(\cdot)$ as in \cref{eq:forward}.
    \label[proposition]{prop:3}
\end{prop}
\cref{prop:3} is derived by using Slutsky's theorem on \cref{prop:2} and offers insight into choosing the stepsize as discussed in \cref{sec:algo_unconstrained}.

\subsection{Black-box Optimization with MBD}
\label{sec:Apd_BBO}

As a zeroth order optimizer, MBD is capable of addressing both trajectory optimization and broader, high-dimensional unconstrained optimization challenges. Such black-box optimization tasks are universally acknowledged as difficult~\cite{akimotoTheoreticalFoundationCMAES2012,yiImprovingSampleEfficiency2024}. We first show superior performance of MBD within this black-box optimization context. In such settings, the Bayesian Optimization technique struggles due to the computational intensity required to develop surrogate models and identify new potential solutions~\cite{erikssonScalableGlobalOptimization2019}. Alternative black-box optimization strategies~\cite{hansenCMAEvolutionStrategy2023} are not limited by computational issues but tend to be less efficient because they do not estimate the black-box function as accurately.  MBD's effectiveness is evaluated using two well-known highly non-convex black-box optimization benchmarks: Ackley~\cite{ackleyConnectionistMachineGenetic1987} and Rastrigin~\cite{balandatBoTorchFrameworkEfficient2020}, each tested across three different dimensionalities. Comparisons were made with CMA-ES~\cite{hansenCMAEvolutionStrategy2023}, TuRBO~\cite{erikssonScalableGlobalOptimization2019}, LA-MCTS~\cite{wangLearningSearchSpace2020}, HesBO~\cite{nayebiFrameworkBayesianOptimization2019}, Shiwa~\cite{liuVersatileBlackBoxOptimization2020}, and BAxUS~\cite{papenmeierIncreasingScopeYou2023}.

\begin{figure}[ht]
    \centering
    \includegraphics[width=1.0\linewidth, bb=0 0 849.6 351.4]{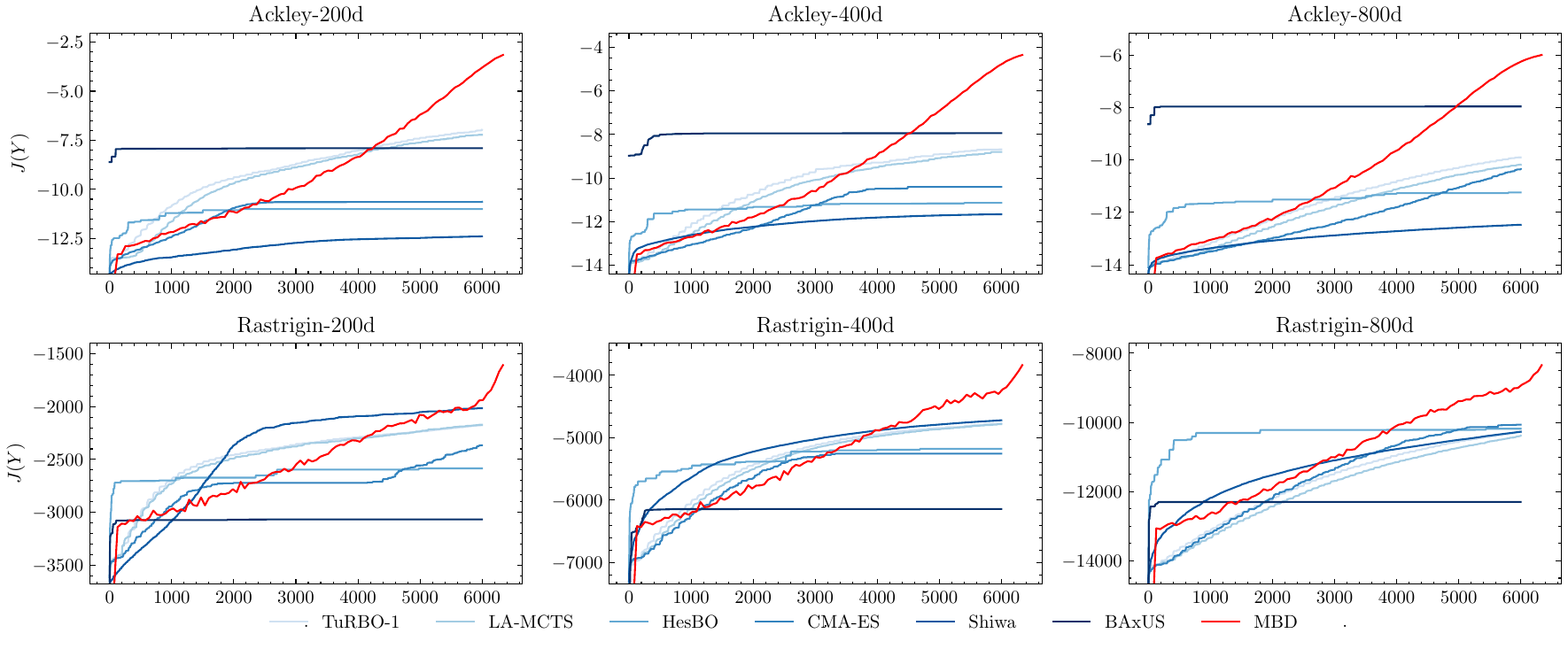}
    \caption{Performance of MBD on high-dimensional black-box optimization benchmarks. \textcolor{red}{MBD} outperforms other \textcolor{blue}{Gaussian Process-based Bayesian Optimization} methods by a clear margin.}
    \label{fig:bbo}
\end{figure}

\cref{fig:bbo} shows the performance of MBD on the Ackley and Rastrigin benchmarks.
MBD demonstrates superior performance over other algorithms for several reasons.
Firstly, the implementation of a scheduled forward process that determines the total number of samples consequently boosts sample efficiency.
Secondly, the application of \SA{} on various $\log{\pYi{i}(\Yi{i})}$ facilitates its escape from local optima of varying scales.
It is important to acknowledge that comparing computational efficiency may not be entirely fair, given that black-box optimization problems typically involve functions that are costly to evaluate.
However, MBD markedly outperforms other Gaussian Process-based Bayesian Optimization approaches, achieving computational time savings of more than twentyfold, similar to the improvements observed with different evolutionary optimization strategies.

Here are the implementation detail for the benchmarks. For the BO benchmarks, the experiments were conducted on an A100 GPU because of the high computational demands of the Gaussian Process Regression Model it incorporates.

\textbf{TuRBO}: TuRBO is implemented based on tutorials from Botorch~\cite{balandatBoTorchFrameworkEfficient2020}.

\textbf{LA-MCTS}: LA-MCTS, we refer to authors’ reference implementations, and use TuRBO as its local BO solver~\cite{wangLearningSearchSpace2020}.

\textbf{HesBO}: For HesBO, we  refer to authors’ reference implementations~\cite{nayebiFrameworkBayesianOptimization2019}. We transformed default GP component into Gpytorch version for faster inference speed on GPU. We set the embedding dimension to 20 for all tasks

\textbf{CMA-ES}: We use pycma\footnote{\href{pycma}{https://github.com/CMA-ES/pycma}
} to implement CMA-ES, and use default setting except setting population size eqauls to batch size.

\textbf{Shiwa}: We use Nevergrad\footnote{\href{nevergrad}{https://github.com/facebookresearch/nevergrad}} to implement Shiwa, and use default setting to run experiments.

\textbf{BAxUS}: We refer to the authors' reference implementations~\cite{papenmeierIncreasingScopeYou2023}.

\subsubsection{MBD for DNN Training without Gradient Information}
To further demonstrate the effectiveness of MBD in high-dimensional systems, we apply MBD to optimize an MLP network for MNIST classification~\cite{MNISTDatabaseHandwritten} without access to the gradient information. 
MBD achieve $85.5\%$ accuracy with $256$ samples within $2$s, which is comparable to the performance of the SGD optimizer with momentum ($92.7\%$ accuracy). 
We use MLP with $2$ hidden layers, each with $32$ neurons, and ReLU activation function. The input is flattened to $784$ dimensions, and the output is a $10$-dimensional vector. 
We use cross-entropy loss as the objective function. The network has $27,562$ parameters in total, which makes sampling-based optimization challenging.
MBD can effectively optimize the network with a small number of samples, demonstrating its effectiveness in high-dimensional black-box optimization tasks.

\subsection{MBD with Demonstration Explaination}
\label{sec:demo_exp}

Data-augmented MBD calculate the score function with demostration as follows:

\begin{align}
                         & \Yi{i-1}                              =  \frac{1}{\sqrt{\alpha_{i}}}\left(\Yi{i} + (1-\bar{\alpha}_{i})\nabla_{\Yi{i}}\log{\pYi{i}(\Yi{i})}\right) \\
                         & \nabla_{\Yi{i}}\log{\pYi{i}(\Yi{i})} = -\frac{\Yi{i}}{1-\bar{\alpha}_{i}} + \frac{\sqrt{\bar{\alpha}_{i}}}{1-\bar{\alpha}_{i}}\bar{Y}^{(0)}        \\
    \mathrm{where} \quad & \bar{Y}^{(0)} = \frac{\sum_{\Yi{0} \in \mathcal{Y}^{(i)}_d} \Yi{0} w(\Yi{0})}{\sum_{\Yi{0} \in \mathcal{Y}^{(i)}_d} w(\Yi{0})}                     \\ &  w(\Yi{0}) = \max{\left\{
        \begin{array}{c}
            w_\text{model}(\Yi{0})=p_d(\Yi{0})p_J(\Yi{0})p_g(\Yi{0}), \\
            w_\text{demo}((\Yi{0}))=p_\text{demo}(\Yi{0})p_J(Y_\text{demo})p_g(Y_\text{demo})
        \end{array}
        \right\}}
\end{align}

\begin{figure}[h!]
    \centering
    \includegraphics[width=1.0\linewidth, bb=0 0 703.395 220.985]{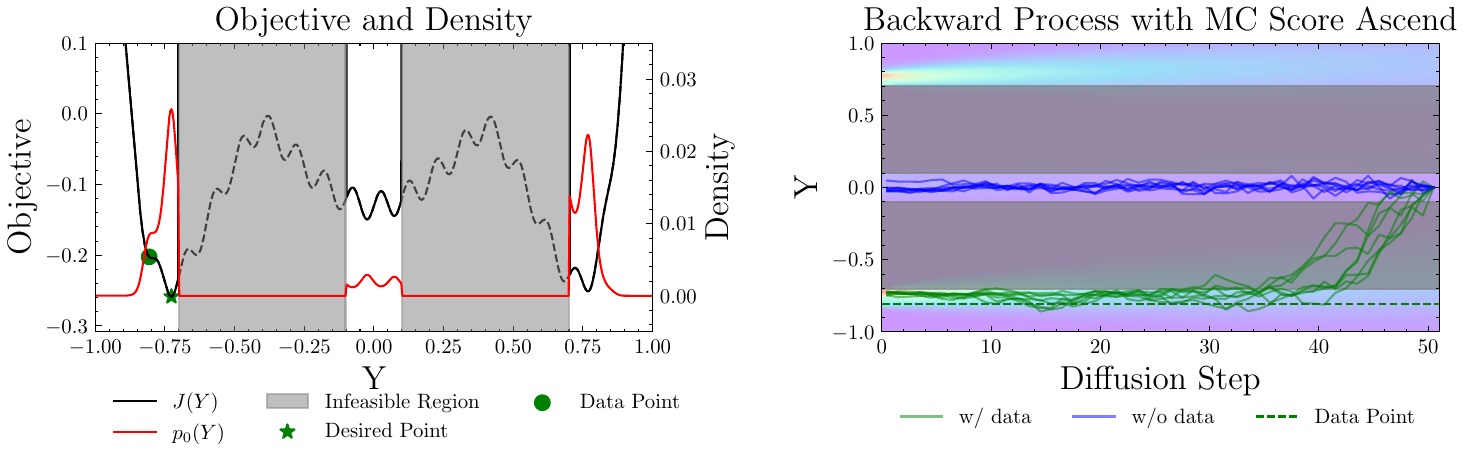}
    \caption{MBD with data vs. without data on a nonconvex function with constraints $||Y|-0.4|>0.3$.
        We want MBD converge to the optimal point \textcolor{green}{\FiveStar} with the help of demonstration data \textcolor{green}{\ding{108}}. Although the demostration point is not optimal, MBD can still converge to the optimal point with the guidance of the demonstration data. Here data serves as a regularization term to guide the diffusion process to the negative optimal point while allowing to use model further to refine the solution.}
    \label{fig:distributions_demo}
\end{figure}

where demonstrate likelihood term $w_\text{demo}(\Yi{0})$ will draw samples towards data without considering the model.
Given $w=w_\text{demo}$, $\bar{Y}^{(0)} = \frac{\sum_{\Yi{0} \in \mathcal{Y}^{(i)}_d} \Yi{0} p_\text{demo}(\Yi{0})}{\sum_{\Yi{0} \in \mathcal{Y}^{(i)}_d} p_\text{demo}(\Yi{0})}=Y_\text{demo}$. The score function would be $\nabla_{\Yi{i}}\log{\pYi{i}(\Yi{i})} = -\frac{\Yi{i}}{1-\bar{\alpha}_{i}} + \frac{\sqrt{\bar{\alpha}_{i}}}{1-\bar{\alpha}_{i}}Y_\text{demo}$, which means the score function is a linear combination of the current sample and the demonstration data.

If we don't use \cref{eq:target_demo} and employ the posterior distribution $p(\Yi{0}|Y_\text{demo}) \propto \pYi{0}(\Yi{0})p_\text{demo}(\Yi{0})$, it will yields update weights $w=w_\text{demo} w_\text{model}$, which will draw samples to both model and demonstration data.
If the demonstration data is not optimal, the final solution will be a compromise between the model and demonstration data.
In~\cref{fig:distributions_demo}, the resulted solution will lie between optimal point \textcolor{green}{\FiveStar} with the help of demonstration data \textcolor{green}{\ding{108}}.

Using the max function in $w$ can aviod this issue. In the early stage while $p_J(\Yi{0})$ is low due to poor sample quality, $w_\text{demo}$ will dominate thanks to the high $p_J(Y_\text{demo})$. This will draw samples towards the demonstration data as shown in the earlier stage of~\cref{fig:distributions_demo}. As the sample quality improves and $p_J(\Yi{0}) > p_J(Y_\text{demo})$, $w_\text{model}$ will dominate and the sample will converge to the optimal point.

\subsection{Experiment Details}
\label{sec:exp_details}

\subsubsection{Simulator and Environment}
\label{sec:env}

We leverage the GPU-accelerated simulator Google Brax~\cite{freemanBraxDifferentiablePhysics2021} to design the locomotion and manipulation tasks.
All task is set to use positional backend in Brax except for the pushT task, which uses the generalizable backend for better contact dynamics simulation. Here we provide a brief description of each task implementations:
\begin{enumerate}
    \item \textbf{Ant}: The Ant task is a 3D locomotion task where the agent is required to move forward as fast as possible. The reward is composed of the forward velocity of the agent and control cost, same as the original Brax implementation. The control dimension is 8.
    \item \textbf{Hopper}: The Hopper task is a 2D locomotion task where the agent is required to jumping forward as fast as possible. We use the same reward function as the original Brax implementation. We modify the simulation substeps from 10 to 20 for longer planning horizon given the same control node. The control dimension is 3.
    \item \textbf{Walker2d}: The Walker2d task is a 2D locomotion task where the agent is required to walk forward. The reward is composed of keep the agent upright and moving forward. The control dimension is 6.
    \item \textbf{Halfcheetah}: The Halfcheetah task is a 2D locomotion task where the agent is required to run forward. The reward is composed of the forward velocity of the agent and control cost. We follow the same reward function as the original Brax implementation. The control dimension is 6.
    \item \textbf{Humanoidrun}: The Humanoidrun task is a 3D locomotion task where the agent is required to run forward. The reward is composed of the forward velocity of the agent and standing upright. Here we also modify the simulation substeps from 10 to 20 for longer planning horizon. The control dimension is 17.
    \item \textbf{Humanoidstandup}: The Humanoidstandup task is a 3D locomotion task where the agent is required to stand up. The reward is the upright torso position of the agent. The control dimension is 17.
    \item \textbf{PushT}: The PushT task is a 2D manipulation task where you can apply force to a sphere to push the T-shaped object to the target location. The reward is composed of the distance between the target and the object and orientation difference between the target and the object. To make the task more challenging, we randomize the target location  20cm away from the initial position and make sure the rotational angle is greater than 135 degrees, which makes it hard to solve the task with single continous contact policy. The control dimension is 2.
    \item \textbf{Car2D}: We implement a 2D car task with standard bicycle dynamics model, where state is $x=[x, y, \theta, v, \delta]$, and action is $u=[a, \delta]$. The dynamics is defined as $\dot{x} = f(x, u) = [v \cos(\theta), v \sin(\theta), \frac{v}{L} \tan(\delta), a, \delta]$.
          The constraints are defined as the U-shape area in the middle of the map, where the car cannot enter.
          The reward is composed of the distance between the target and the car and the control cost.
          The control dimension is 2.
\end{enumerate}

\subsubsection{MBD Hyperparameters}
\label{sec:mbd_impl}

In general, MBD is very little hyperparameters to tune compared with RL. We use the same hyperparameters for all the tasks, with small tweaks for harder tasks.

\begin{table}[ht]
    \centering
    \begin{tabular}{@{}llll@{}}
        \toprule
        Task Name       & Horizon & Sample Number & Temperature \(\lambda\) \\
        \midrule
        Ant             & 50      & 100           & 0.1                     \\
        Halfcheetah     & 50      & 100           & 0.4                     \\
        Hopper          & 50      & 100           & 0.1                     \\
        Humanoidstandup & 50      & 100           & 0.1                     \\
        Humanoidrun     & 50      & 300           & 0.1                     \\
        Walker2d        & 50      & 100           & 0.1                     \\
        PushT           & 40      & 200           & 0.2                     \\
        \bottomrule
    \end{tabular}
    \caption{MBD hyperparameters for various tasks}
    \label{tab:task_params}
\end{table}

For diffusion noise schedulling, we use simple linear scheduling $\beta_0 = 1 \times 10^{-4}$ and $\beta_N = 1 \times 10^{-2}$, and the diffusion step number is $100$ across all tasks. Each step's $\alpha_i$ is calculated as $\alpha_i = 1 - \beta_i$. 

\subsubsection{Baseline Algorithms Implementation}
\label{sec:baseline}

For reinforcement learning implementation, we strictly follow the hyperparameters and implementation details provided by the original Brax repository, which optimize for the best performance. For our self-implemented PushT task, the hyperparameters is ported from Pusher task in Brax for fair comparison. The hyperparameters for the RL tasks are shown in \cref{tab:rl_config} and \cref{tab:rl_specifics}.

\begin{table}[h]
    \centering
    \begin{tabular}{@{}lllll@{}}
        \toprule
        Environment     & Algorithm & Timesteps & Reward Scaling & Episode Length \\
        \midrule
        Ant             & PPO       & 100M      & 10             & 1000           \\
        Hopper          & SAC       & 6.55M     & 30             & 1000           \\
        Walker2d        & PPO       & 50M       & 1              & 1000           \\
        Halfcheetah     & PPO       & 50M       & 1              & 1000           \\
        Pusher          & PPO       & 50M       & 5              & 1000           \\
        PushT           & PPO       & 100M      & 1.0            & 100            \\
        Humanoidrun     & PPO       & 100M      & 0.1            & 100            \\
        Humanoidstandup & PPO       & 100M      & 0.1            & 1000           \\
        \bottomrule
    \end{tabular}
    \caption{General RL configuration for various environments}
    \label{tab:rl_config}
\end{table}

\begin{table}[h!]
    \centering
    \begin{tabular}{@{}lllll@{}}
        \toprule
        Environment     & Minibatches & Updates/Batch & Discounting & Learning Rate      \\
        \midrule
        Ant             & 32          & 4             & 0.97        & $3 \times 10^{-4}$ \\
        Hopper          & 32          & 4             & 0.997       & $6 \times 10^{-4}$ \\
        Walker2d        & 32          & 8             & 0.95        & $3 \times 10^{-4}$ \\
        Halfcheetah     & 32          & 8             & 0.95        & $3 \times 10^{-4}$ \\
        Pusher          & 16          & 8             & 0.95        & $3 \times 10^{-4}$ \\
        PushT           & 16          & 8             & 0.99        & $3 \times 10^{-4}$ \\
        Humanoidrun     & 32          & 8             & 0.97        & $3 \times 10^{-4}$ \\
        Humanoidstandup & 32          & 8             & 0.97        & $6 \times 10^{-4}$ \\
        \bottomrule
    \end{tabular}
    \caption{RL specifics for various environments}
    \label{tab:rl_specifics}
\end{table}

For the zeroth order optimization tasks, we the same hyperparameters as the MBD algorithm.

\subsubsection{Demonstration Collections}

For RRT algorithm in Car2D task, we set the max step size to $0.2$, and the max iterations to $1000$ given the maximum episode length is $50$.

For the demonstration collection in Humanoid Jogging task, we first download the mocap data which contains each joints' position in the world frame. Then we use the joint data to calculate the position of torso, thigh and shin position as partial state reference for our task.

\end{document}